\documentclass[a4paper,UKenglish]{article}

\usepackage{amsmath,amssymb}
\usepackage{eurosym}

\usepackage{amssymb}
\usepackage{epstopdf}

\usepackage{amsthm}
\usepackage{tikz}
\usetikzlibrary{fit,calc,positioning,decorations.pathreplacing,matrix, shapes}
\usepackage{multirow}
\usepackage{multicol}
\usepackage{graphicx}
\usepackage{latexsym}
\usepackage[lined,linesnumbered,boxed]{algorithm2e}
\usepackage{enumerate}
\usepackage{float}
\usepackage{caption,subcaption}
\usepackage{array}

\newcommand{\A}{{\sf \bf A}}
\newcommand{\R}{{\sf \bf R}}
\newcommand{\SAT}{{{\sc SAT}}}
\newcommand{\NP}{{\sf NP}}

\newcommand{\Control}{{\sf Control}}
\newcommand{\PSA}{{\sf PSA}}
\newcommand{\PCA}{{\sf PCA}}
\newcommand{\coNP}{{\sf coNP}}
\DeclareMathOperator{\gr}{gr}
\DeclareMathOperator{\stb}{st}
\DeclareMathOperator{\pr}{pr}
\DeclareMathOperator{\co}{co}
\DeclareMathOperator{\sk}{sk}
\DeclareMathOperator{\cred}{cr}

\makeatletter
\DeclareRobustCommand*\cal{\@fontswitch\relax\mathcal}
\makeatother

\newcommand{\caf}{\ensuremath{{\cal CAF}}}
\newcommand{\iaf}{\ensuremath{{\cal IAF}}}
\newcommand{\af}{\ensuremath{{\cal AF}}}

\newtheorem{theorem}{Theorem}
\newtheorem{proposition}[theorem]{Proposition}
\newtheorem{lemma}[theorem]{Lemma}
\newtheorem{observation}[theorem]{Observation}
\newtheorem{definition}{Definition}
\newtheorem{example}{Example}

\begin{document}
\title{Possible Controllability of Control Argumentation Frameworks -
  Extended Version}
\author{Jean-Guy Mailly\\LIPADE, Universit\'e de Paris,
  France\\jean-guy.mailly@u-paris.fr}
\date{}

\maketitle

\begin{abstract}
  The recent Control Argumentation Framework (CAF) is a
  generalization of Dung's Argumentation Framework which handles
  argumentation dynamics under uncertainty; especially it can be used
  to model the behavior of an agent which can anticipate future
  changes in the environment. Here we provide new insights on this
  model by defining the notion of possible controllability of a CAF. We study the complexity
  of this new form of reasoning for the four classical semantics, and we
  provide a logical encoding for reasoning with this framework.
\end{abstract}

Note: This paper has been accepted for publication at COMMA'20. Proofs were omitted from the published version. For refering this work, please cite:\\
\noindent\fbox{%
  \parbox{0.97\textwidth}{%
    Jean-Guy \textsc{Mailly}. \emph{Possible Controllability of
      Control Argumentation Frameworks}. Proceedings of the Eighth
    International Conference on Computational Models of Argument
    (COMMA'20): 283--294, 2020.
  }
}

\section{Introduction}
Abstract argumentation \cite{Dung95} has become an important subfield
of Knowledge Representation and Reasoning research in the last
decades. Intuitively, an abstract argumentation framework (AF) is a
directed graph where nodes are arguments and edges are relations
(usually attacks) between these arguments. The outcome of such an AF
is an evaluation of the arguments' acceptance (through extensions
\cite{Dung95,HOFASemantics}, labellings \cite{Caminada06} or rankings
\cite{AmgoudB13a}). The question of argumentation dynamics has arisen
more recently, and many different approaches have been proposed (see
{\em e.g.}
\cite{BoellaKT09a,CayrolSL10,BaumannB10,Coste-MarquisKMM14,DoutreHP14,Coste-MarquisKM15,Saint-CyrBCL16,WallnerNJ17,DoutreM17,DimopoulosMM18}).
Roughly speaking, the question of these works is ``how to modify an AF
to be consistent with a given piece of information?''. Such a piece
of information can be ``argument $a$ should be accepted in the outcome
of the AF''. A particular version of this problem is called {\em
  extension enforcement} \cite{BaumannB10,Baumann12,
  Coste-MarquisKM15, WallnerNJ17,DoutreM17}: it consists in modifying an AF s.t. a
given set of arguments becomes (included in) an extension of the
AF. The recently proposed {\em Control Argumentation Framework} (CAF)
\cite{DimopoulosMM18} is a generalization of Dung's AF which
incorporates different notions of uncertainty in the structure of the
framework. The {\em controllability} of a CAF w.r.t. a set of
arguments is the fact that, whatever happens in the uncertain part of
the CAF ({\em i.e.}  whatever is the real situation of the world), the
target set of arguments is accepted. This is somehow a generalization
of extension enforcement, where uncertainty is taken into account.

In this paper, we study what we call {\em possible controllability} (and
then, controllability defined in \cite{DimopoulosMM18} can be renamed
as {\em necessary controllability}). The idea of possible controllability
w.r.t. a target set of arguments is that this target should be
accepted in {\em at least} one of the possible completions of the
uncertain part. Necessary controllability trivially implies possible
controllability, while the converse is not true. 
This form of reasoning can be applied in different situations. Possible
controllability makes sense in situations where an agent is unable to
guarantee some result (the fact that some argument $a$ is accepted),
but she wants to be sure that the opposite result ($a$ is rejected) is
not necessary true. For instance, possible controllability is similar to
the reasoning of the defendant's lawyer during a trial. Thanks to the
principle of {\em presumption of innocence}, the lawyer does not have
to prove that the defendant {\em is} innocent, but he has to prove
that the defendant {\em may be} innocent. This means that if there is
some uncertainty in the case, the lawyer wants to exhibit the fact
that one possible world encompassed by this uncertainty implies that
his client is innocent.\footnote{On the opposite, necessary
  controllability \cite{DimopoulosMM18} is close to the reasoning
 of the prosecutor.} This means that the lawyer's knowledge
about the case can be represented by a CAF, and the lawyer wants to
guarantee that the argument ``the defendant is innocent'' is accepted
in at least one completion of the CAF, {\em i.e.} one possible world.
In this kind of scenario, possible controllability is particularly useful since it is (presumably) easier to
search for one completion that accepts the target instead of checking that the target is accepted
in each of the (exponentially many) completions.

The paper is organized as follows. We first recall the background
notions of logic and introduce the CAF setting in
Section~\ref{section:background}.  In
Section~\ref{section:weak-controllability} we define formally this new
form of controllability, and we determine the complexity of this reasoning problem
for the four classical semantics introduced by Dung. We also propose a QBF-based
encoding which allows to determine whether a CAF is possible controllable w.r.t. a target  and the stable semantics (and moreover, which allows to determine {\em how} to control it). We describe the
related work in Section~\ref{section:related-work}, and finally
Section~\ref{section:conclusion} concludes the paper and draws
interesting future research tracks.

\section{Background}\label{section:background}
\subsection{Propositional Logic and Quantified Boolean Formulas}
We consider a set $V$ of Boolean variables, {\em i.e.} variables which
can be assigned a value in $\mathbb{B} = \{0,1\}$, where $0$ and $1$
are associated respectively with {\em false} and {\em true}. Such
variables can be combined with connectives $\{\vee,\wedge,\neg\}$ to
build formulas. $x \vee y$ is true if at least one of the variables
$x, y$ is true; $x \wedge y$ is true if both $x, y$ are true; $\neg x$
is true is $x$ is false. Additional connectives can be defined, {\em
  e.g.}  $x \Rightarrow y$ is equivalent to $\neg x \vee y$;
$x \Leftrightarrow y$ is equivalent to
$(x \Rightarrow y) \wedge (y \Rightarrow x)$.  The definition of the
connectives is straightforwardly extended from variables to formulas
({\em e.g.} if $\phi$ and $\psi$ are formulas, then $\phi \wedge \psi$
is true when both formulas are true). A truth assignment on the set of
variables $V = \{x_1, \dots, x_n \}$ is a mapping
$\omega: V \rightarrow \mathbb{B}$.

Quantified Boolean Formulas (QBFs) are an extension of propositional
formulas with the universal and existential quantifiers. For instance,
the formula
$\exists x \forall y (x \vee \neg y) \wedge (\neg x \vee y)$ is
satisfied if there is a value for $x$ such that for all values of $y$
the proposition $(x \vee \neg y) \wedge (\neg x \vee y)$ is true. More
formally, a canonical QBF is a formula
${\cal Q}_1 X_1 {\cal Q}_2 X_2 \ldots {\cal Q}_n X_n \Phi$ where
$\Phi$ is a propositional formula,
${\cal Q}_i \in \{ \exists, \forall\}$,
${\cal Q}_i \not = {\cal Q}_{i+1}$, and $X_1, X_2, \ldots, X_n$
disjoint sets of propositional variables such that
$X_1 \cup X_2 \cup \ldots \cup X_n = V$.\footnote{If some variable
  $x \in V$ does not explicitly belong to any $X_i$, {\em i.e.}
  $X_1 \cup \dots \cup X_n \subset V$, then it implicitly means that
  $x$ can be existentially quantified at the rightmost level.} It is
well-known that QBFs span the polynomial hierarchy. For instance,
deciding whether the formula
$\exists X_1 \forall X_2 \dots {\cal Q}_i X_i \Phi$ is true is
$\Sigma^p_i$-complete.
The decision problem associated to QBFs of the form
$\exists V, \Phi$ is equivalent to the satisfiability problem for
propositional formulas (\SAT), which is well-known to be \NP-complete.
For more details about propositional logic, QBFs and complexity
theory, we refer the reader to
\cite{BiereHVMW09,QBFHandbook,AroraB09}.

\subsection{Abstract Argumentation and Control Argumentation Frameworks}
An {\em argumentation framework} (AF), introduced in~\cite{Dung95}, is
a directed graph $\af = \langle A, R \rangle$, where $A$ is
a set of {\em arguments}, and $R \subseteq A \times A$ is an {\em
  attack relation}.  The relation $a$ {\em attacks} $b$ is denoted by
$(a,b) \in R$. In this setting, we are not interested in the origin of
arguments and attacks, nor in their internal structure. Only their
relations are important to define the acceptance of arguments.

In \cite{Dung95}, different acceptability semantics were
introduced. They are based on two basic concepts: \textit{conflict-freeness} and
\textit{defence}. A set $S \subseteq A$ is:
\begin{itemize}
	\item conflict-free iff $\forall a, b \in S$, $(a,b) \not\in R$;
	\item admissible iff it is conflict-free, and defends each $a \in S$ against its attackers.
\end{itemize}
The semantics defined by Dung are as follows. An admissible set $S \subseteq A$ is:
\begin{itemize}
	\item a complete extension iff it contains every argument that it defends;
	\item a preferred extension iff it is a $\subseteq$-maximal complete extension;
	\item the unique grounded extension iff it is the $\subseteq$-minimal complete extension;
	\item a stable extension iff it attacks every argument in $A \setminus S$.
\end{itemize}

The sets of extensions of an $\af$, for these four semantics, are denoted (respectively) $\co(\af)$, $\pr(\af)$, $\gr(\af)$ and $\stb(\af)$.

Our approach could be adapted for any other extension semantics. Based on these
semantics, we can define the status of any (set of) argument(s), namely
{\em skeptically accepted} (belonging to each $\sigma$-extension),
{\em credulously accepted} (belonging to some $\sigma$-extension) and
{\em rejected} (belonging to no $\sigma$-extension). For
more details about argumentation semantics, we refer the reader to
\cite{Dung95,HOFASemantics}.\\

We introduce now the notions of CAF and (necessary) controllability \cite{DimopoulosMM18}.

\begin{definition}\label{def:caf}
A Control Argumentation Framework ({\em CAF}) is a triple
  $\caf = \langle {\cal F}, {\cal C}, {\cal U}\rangle$ where $\cal F$
  is the {\em fixed part}, $\cal U$ is the {\em uncertain part} and
  $\cal C$ is the {\em control part} of $\caf$ with:

\begin{itemize}
\item $\cal F$ = $\langle A_F, \rightarrow \rangle$ where $A_F$ is a
  set of arguments and
  $\rightarrow \subseteq (A_F \cup A_U) \times (A_F \cup A_U)$ is an
  attack relation.
\item $\cal U$ =
  $\langle A_U, (\rightleftarrows \cup \dashrightarrow) \rangle$ where
  $A_U$ is a set of arguments,
  $\rightleftarrows \subseteq (((A_U \cup A_F) \times (A_U \cup A_F))
  \setminus \rightarrow)$ is a conflict relation and
  $ \dashrightarrow \subseteq (((A_U \cup A_F) \times (A_U \cup A_F))
  \setminus \rightarrow)$ is an attack relation, with
  $\rightleftarrows \cap \dashrightarrow = \emptyset$.
\item ${\cal C} = \langle A_C, \Rrightarrow \rangle$ where $A_C$ is a
  set of arguments, and
  $\Rrightarrow \subseteq \{(a_i,a_j) \mid a_i \in A_C,\ a_j \in A_F
  \cup A_C \cup A_U\}$ is an attack relation.
\end{itemize}
$A_F, A_U$ and $A_C$ are disjoint subsets of arguments.
\end{definition}

The different sets of arguments and attacks have different
meanings. The fixed part ${\cal F}$ represents the part of the system
which cannot be influenced either by the agent or by the
environment. This means that if $a \in A_F$, then it is sure that $a$
is an ``active'' argument (for instance, all of its premises are true,
and cannot be falsified). Similarly, if $(a,b) \in \rightarrow$, the
attack from $a$ to $b$ is actually part of the system and cannot be
removed.

${\cal U}$ is the uncertain part of the system. This means that it
cannot be influenced by the agent, but it can be modified by the
environment (in a wide way, this can also represent the possible
actions of other agents). The uncertainty can appear in different
ways. First, if $a \in A_U$, this means that there is some uncertainty
about the presence of an argument (for instance, the agent is not sure
whether her opponent in the debate will state argument $a$, or she is
not sure whether the premises of $a$ will be true at some moment). If
$(a,b) \in \rightleftarrows$, then the agent is sure that there is a
conflict between $a$ and $b$, but she is not sure of the direction of
the attack (this could be an attack $(a,b)$, an attack $(b,a)$, or
even both at the same time). This is possible, for instance, if the
agent is not sure about some preference between $a$ and $b$
\cite{AmgoudV14}. Finally, $(a,b) \in \dashrightarrow$ means that the
agent is not sure whether there is actually an attack from $a$ to $b$.

 The last part ${\cal C}$ is the {\em control} part. This is the part
 of the system which can be influenced by the agent. This means that
 the agent has to choose which arguments she will actually use
 (uttering them in the debate, or making an action to switch their
 premises to true). When the agent uses a subset
 $A_{conf} \subseteq A_C$, called a {\em configuration}, this defines
 a configured CAF where the arguments from $A_C \setminus A_{conf}$
 (and the attacks concerning them) are removed. We illustrate these
 concepts on an example adapted from \cite{DimopoulosMM18}.

\begin{example}\label{example:running-caf}
We define $\caf = \langle F, C, U\rangle$ as follows:
\begin{itemize}
  \item ${\cal F} = \langle \{a_1,a_2,a_3,a_4, a_5\}, \{(a_2,a_1), (a_3,a_1), (a_4,a_2), (a_4,a_3)\}\rangle$;
  \item ${\cal U} = \langle \{a_6\}, \rightleftarrows \cup \dashrightarrow \rangle$, with $\rightleftarrows = \{(a_6,a_4)\}$, and $\dashrightarrow = \{(a_5,a_1)\}$;
  \item ${\cal C} = \langle \{a_7,a_8,a_9\}, \{(a_7,a_5),(a_7,a_9), (a_8,a_6), (a_8,a_7), (a_9,a_6)\}\rangle$.
\end{itemize}
$\caf$ is given at Figure~\ref{fig:example-caf}. The configuration of
$\caf$ by $A_{conf} = \{a_7,a_9\}$ yields the configured CAF $\caf'$
described at Figure~\ref{fig:configured-caf}. On the figures,
arguments from $A_F$, $A_U$ and $A_C$ are respectively represented as
circle nodes, dashed square nodes and plain square nodes. Similarly,
the attacks from $\rightarrow$, $\rightleftarrows$,
$\dashrightarrow$ and $\Rrightarrow$ are represented (respectively) as plain,
double-headed dashed, dotted and bold arrows.

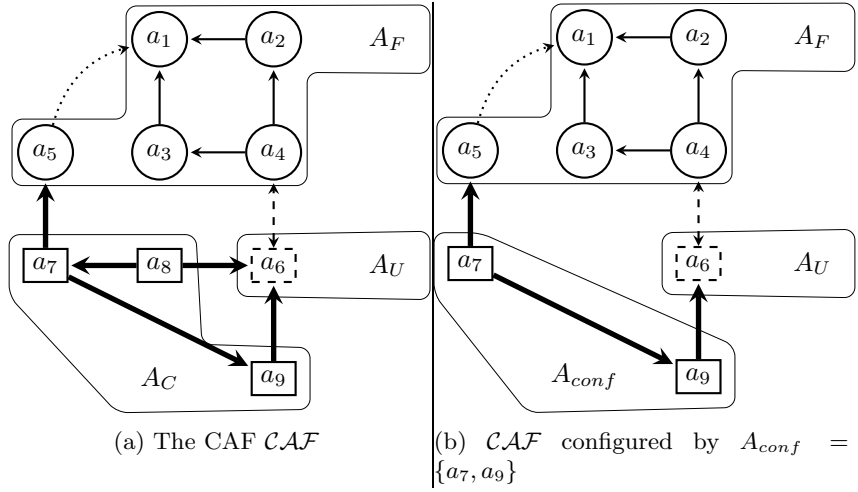
\begin{figure}[h]
  \centering
  \begin{subfigure}[t]{0.45\textwidth}
    \centering
\newcommand{\shiftpoints}{5pt}
\begin{tikzpicture}[->,>=stealth,shorten >=1pt,auto,node distance=1.5cm,
                thick,main node/.style={circle,draw,font=\bfseries},control/.style={rectangle,draw,font=\bfseries},uncertain/.style={rectangle,draw,dashed,font=\bfseries},shifttl/.style={shift={(-\shiftpoints,\shiftpoints)}},shifttr/.style={shift={(\shiftpoints,\shiftpoints)}},shiftbl/.style={shift={(-\shiftpoints,-\shiftpoints)}},shiftbr/.style={shift={(\shiftpoints,-\shiftpoints)}}]
\node[main node] (a1) {$a_1$};
\node[main node] (a2) [right of=a1] {$a_2$};
\node[main node] (a3) [below of=a1] {$a_3$};
\node[main node] (a4) [below of=a2] {$a_4$};
\node[main node] (a5) [left of=a3] {$a_5$};
\node[right of=a2] (f) {$A_F$};

\path[->] (a2) edge (a1)
    (a3) edge (a1)
    (a4) edge (a2)
    (a4) edge (a3);

\path[draw,rounded corners,very thin]
               ([shifttl] a1.north west)
            -- ([shifttr] f.north east)
            -- ([shiftbr] f.south east)
            -- ([shiftbr] a2.south east)
            -- ([shiftbr] a4.south east)
            -- ([shiftbl] a3.south west)
            -- ([shiftbl] a5.south west)
            -- ([shifttl] a5.north west)
            -- ([shifttl] a3.north west)
            -- ([shiftbl] a1.south west)
            -- cycle;

\node[uncertain] (a6) [below of=a4] {$a_6$};
\node[right of=a6] (u) {$A_U$};

\path[->,dotted] (a5) edge[bend left] (a1);
\path[<->,dashed] (a6) edge (a4);

\path[draw,rounded corners,very thin]
                ([shifttr] u.north east)
             -- ([shifttl] a6.north west)
             -- ([shiftbl] a6.south west)
             -- ([shiftbr] u.south east)
             -- cycle ;

\node[control] (a7) [below of=a5] {$a_7$};
\node[control] (a8) [left of=a6] {$a_8$};
\node[control] (a9) [below of=a6] {$a_9$};
\node[below of=a8] (c) {$A_C$};

\path[line width=2pt] (a7) edge (a5)
     (a8) edge (a7)
     (a8) edge (a6)
     (a9) edge (a6)
     (a7) edge (a9);

\path[draw,rounded corners,very thin]
                ([shifttr] a9.north east)
             -- ([shifttr] c.north east)
             -- ([shifttr] a8.north east)
             -- ([shifttl] a7.north west)
             -- ([shiftbl] a7.south west)
             -- ([shiftbl] c.south west)
             -- ([shiftbr] a9.south east)
             -- cycle;
\end{tikzpicture}
\caption{The CAF $\caf$\label{fig:example-caf}}
\end{subfigure}
\vline
\begin{subfigure}[t]{0.45\textwidth}
  \centering
\newcommand{\shiftpoints}{5pt}
\begin{tikzpicture}[->,>=stealth,shorten >=1pt,auto,node distance=1.5cm,
                thick,main node/.style={circle,draw,font=\bfseries},control/.style={rectangle,draw,font=\bfseries},uncertain/.style={rectangle,draw,dashed,font=\bfseries},shifttl/.style={shift={(-\shiftpoints,\shiftpoints)}},shifttr/.style={shift={(\shiftpoints,\shiftpoints)}},shiftbl/.style={shift={(-\shiftpoints,-\shiftpoints)}},shiftbr/.style={shift={(\shiftpoints,-\shiftpoints)}}]
\node[main node] (a1) {$a_1$};
\node[main node] (a2) [right of=a1] {$a_2$};
\node[main node] (a3) [below of=a1] {$a_3$};
\node[main node] (a4) [below of=a2] {$a_4$};
\node[main node] (a5) [left of=a3] {$a_5$};
\node[right of=a2] (f) {$A_F$};

\path[->] (a2) edge (a1)
    (a3) edge (a1)
    (a4) edge (a2)
    (a4) edge (a3);

\path[draw,rounded corners,very thin]
               ([shifttl] a1.north west)
            -- ([shifttr] f.north east)
            -- ([shiftbr] f.south east)
            -- ([shiftbr] a2.south east)
            -- ([shiftbr] a4.south east)
            -- ([shiftbl] a3.south west)
            -- ([shiftbl] a5.south west)
            -- ([shifttl] a5.north west)
            -- ([shifttl] a3.north west)
            -- ([shiftbl] a1.south west)
            -- cycle;

\node[uncertain] (a6) [below of=a4] {$a_6$};
\node[right of=a6] (u) {$A_U$};

\path[->,dotted] (a5) edge[bend left] (a1);
\path[<->,dashed] (a6) edge (a4);

\path[draw,rounded corners,very thin]
                ([shifttr] u.north east)
             -- ([shifttl] a6.north west)
             -- ([shiftbl] a6.south west)
             -- ([shiftbr] u.south east)
             -- cycle ;

\node[control] (a7) [below of=a5] {$a_7$};
\node[control] (a9) [below of=a6] {$a_9$};
\node[below of=a8] (c) {$A_{conf}$};

\path[line width=2pt] (a7) edge (a5)
     (a9) edge (a6)
     (a7) edge (a9);

\path[draw,rounded corners,very thin]
                ([shifttr] a9.north east)
             -- ([shifttl] a7.north east)
             -- ([shifttl] a7.north west)
             -- ([shiftbl] a7.south west)
             -- ([shiftbl] c.south west)
             -- ([shiftbr] a9.south east)
             -- cycle;
\end{tikzpicture}
\caption{$\caf$ configured by $A_{conf} = \{a_7,a_9\}$ \label{fig:configured-caf}}
\end{subfigure}
\caption{A CAF and a configured CAF}
\end{figure}
\end{example}

Now we recall the notion of completion, borrowed from
\cite{Coste-MarquisDKLM07}, and adapted to CAFs in
\cite{DimopoulosMM18}. Intuitively, a completion is a classical AF
which describes a situation of the world coherent with the uncertain
information encoded in the CAF.

\begin{definition}  \label{compldef}
  Given $\caf = \langle F, C, U \rangle$, a completion of
  $\caf$ is $\af$ = $\langle A, R \rangle$, s.t.
  \begin{itemize}
    \item $A = A_F \cup A_C \cup A_{comp}$ where $A_{comp} \subseteq A_U$;
    \item if $(a,b) \in R$, then $(a,b) \in \rightarrow \cup \rightleftarrows \cup \dashrightarrow \cup \Rrightarrow$;
    \item if $(a,b) \in \rightarrow$, then $(a,b) \in R$;
    \item if $(a,b) \in \rightleftarrows$ and $a,b \in A$, then $(a,b) \in R$ or $(b,a) \in R$;
    \item if $(a,b) \in \Rrightarrow$ and $a,b \in A$, then $(a,b) \in R$.
  \end{itemize}
\end{definition}

\begin{example}[Continuation of Example~\ref{example:running-caf}]
We describe two possible completions of $\caf'$. First, we consider a
completion $\af_1$ where the attack $(a_5,a_1)$ is not included, while
the argument $a_6$ (with the attack $(a_6,a_4)$) is included. Another
possible completion is $\af_2$, where $a_6$ is not included (so,
neither the attacks related to it) while the attack $(a_5,a_1)$ is included.

\begin{figure}[h]
  \centering
  \begin{subfigure}[t]{0.45\textwidth}
    \centering
\begin{tikzpicture}[->,>=stealth,shorten >=1pt,auto,node distance=1.5cm,
                thick,main node/.style={circle,draw,font=\bfseries}]
\node[main node] (a1) {$a_1$};
\node[main node] (a2) [right of=a1] {$a_2$};
\node[main node] (a3) [below of=a1] {$a_3$};
\node[main node] (a4) [below of=a2] {$a_4$};

\path[->] (a2) edge (a1)
    (a3) edge (a1)
    (a4) edge (a2)
    (a4) edge (a3);

\node[main node] (a5) [below of=a3] {$a_5$};
\node[main node] (a6) [right of=a5] {$a_6$};

\path[->] (a6) edge (a4);

\node[main node] (a7) [left of=a5] {$a_7$};
\node[main node] (a9) [right of=a6] {$a_9$};

\path[->] (a7) edge (a5)
     (a9) edge (a6)
     (a7) edge[bend right] (a9);

\end{tikzpicture}
\caption{$\af_1$}
\end{subfigure}
\vline
\begin{subfigure}[t]{0.45\linewidth}
  \centering
\begin{tikzpicture}[->,>=stealth,shorten >=1pt,auto,node distance=1.5cm,
                thick,main node/.style={circle,draw,font=\bfseries}]
\node[main node] (a1) {$a_1$};
\node[main node] (a2) [right of=a1] {$a_2$};
\node[main node] (a3) [below of=a1] {$a_3$};
\node[main node] (a4) [below of=a2] {$a_4$};

\path[->] (a2) edge (a1)
    (a3) edge (a1)
    (a4) edge (a2)
    (a4) edge (a3);

\node[main node] (a5) [below of=a3] {$a_5$};

\path[->] (a5) edge[bend left] (a1);

\node[main node] (a7) [left of=a5] {$a_7$};
\node[main node] (a9) [right of=a5] {$a_9$};

\path[->] (a7) edge (a5)
     (a7) edge[bend right] (a9)
     ;

\end{tikzpicture}
\caption{$\af_2$}
\end{subfigure}
\caption{Two possible completions of $\caf'$}
\end{figure}
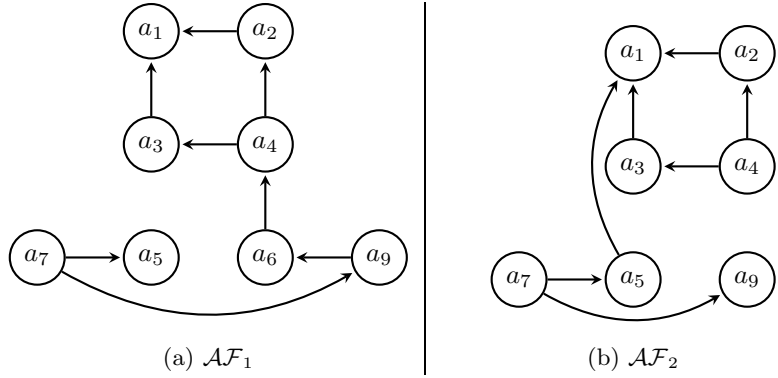
\end{example}

Now, a CAF is necessary controllable w.r.t. a target $T \subseteq A_F$
if the agent can configure it in a way which guarantees that $T$ is
accepted in every completion of the configured CAF. This necessary
controllability has two versions, depending on the kind of acceptance
under consideration (skeptical or credulous).

\begin{definition}
  Given a set of arguments $T \subseteq A_F$ and a semantics $\sigma$,
  $\caf$ is {\em necessary skeptically (resp. credulously)
    controllable} 
     w.r.t. $T$ and $\sigma$ iff $\exists A_{conf} \subseteq A_C$ s.t. $T$ is included
     in each (resp. some) $\sigma$-extension of each completion of
  $\caf' = \langle F, C', U\rangle$, with
  $C' = \langle A_{conf}, \{(a_i,a_j) \in \Rrightarrow \mid
  a_i, a_j \in (A_F \cup A_U \cup A_{conf})\}\rangle$.
\end{definition}

\cite{DimopoulosMM18} proposes a QBF-based method to determine whether
a CAF is necessary controllable, and to obtain the corresponding
configuration if it exists.

\section{Possible Controllability}\label{section:weak-controllability}
\subsection{Formal Definition of Possible Controllability}
The intuition of necessary controllability is that the agent is
satisfied when its target is reached in every possible world encoded
by the uncertain information in the CAF. While this is an interesting
property (especially for applications like negotiation
\cite{DimopoulosMM19}), this may seem unrealistic for some
applications, where the graph is built in such a way that some
completions cannot accept the target. Here, we adapt the definition of
controllability to consider that the agent is satisfied whether there
exists at least one possible world ({\em i.e.}  one completion) which
accepts the target.

\begin{definition}
  Given a set of arguments $T \subseteq A_F$ and a semantics $\sigma$,
 $\caf$ is {\em possibly skeptically (resp. credulously) controllable}
 w.r.t. $T$ and $\sigma$ iff $\exists A_{conf} \subseteq A_C$ s.t. $T$ is included
  in each (resp. some) $\sigma$-extension of some completion
  of $\caf' = \langle F, C', U\rangle$, with
  $C' = \langle A_{conf}, \{(a_i,a_j) \in \Rrightarrow \mid
  a_i, a_j \in (A_F \cup A_U \cup A_{conf})\}\rangle$.
\end{definition}

\begin{observation}
  Given a set of arguments $T \subseteq A_F$ and a semantics $\sigma$,
  if $\caf$ is necessary skeptically (resp. credulously) controllable
  w.r.t. $T$ and $\sigma$, then $\caf$ is possibly skeptically
  (resp. credulously) controllable w.r.t. $T$ and $\sigma$. The
  converse is false.
\end{observation}

\begin{example}[Continuation of Example~\ref{example:running-caf}]
We observe that  $\caf$ from the previous example is not
necessary skeptically controllable w.r.t. the target $\{a_1\}$. Indeed,
\begin{itemize}
\item if $A_{conf} = \{a_7, a_8, a_9\}$, then because of the attack
  $(a_8, a_7)$, the target is not defended against the potential
  threat $(a_5,a_1) \in \dashrightarrow$. The same thing happens if
  $A_{conf} = \{a_7, a_8\}$ or   $A_{conf} = \{a_8, a_9\}$.
\item if $A_{conf} = \{a_7, a_9\}$, this time the target is not
  defended against the potential threat coming from $a_6$ (in the
  completions where $a_6$ belongs to the system, along with the attack
  $(a_6,a_4)$, $a_1$ is not accepted).
\item if $A_{conf}$ is one of the three possible singletons, then
  again $a_1$ is not accepted in every completion (since either $a_5$
  or $a_6$ is unattacked).
\end{itemize}
On the opposite, it is possible to configure $\caf$ is such a way that
$a_1$ is skeptically accepted in at least one completion. For
instance,
Figure~\ref{fig:success-configured-caf} describes such a configured CAF, with
a successful completion given at
Figure~\ref{fig:success-configured-completion}.

\begin{figure}[h]
  \centering
  \begin{subfigure}[t]{0.45\linewidth}
    \centering
\newcommand{\shiftpoints}{5pt}
\begin{tikzpicture}[->,>=stealth,shorten >=1pt,auto,node distance=1.5cm,
                thick,main node/.style={circle,draw,font=\bfseries},control/.style={rectangle,draw,font=\bfseries},uncertain/.style={rectangle,draw,dashed,font=\bfseries},shifttl/.style={shift={(-\shiftpoints,\shiftpoints)}},shifttr/.style={shift={(\shiftpoints,\shiftpoints)}},shiftbl/.style={shift={(-\shiftpoints,-\shiftpoints)}},shiftbr/.style={shift={(\shiftpoints,-\shiftpoints)}}]
\node[main node] (a1) {$a_1$};
\node[main node] (a2) [right of=a1] {$a_2$};
\node[main node] (a3) [below of=a1] {$a_3$};
\node[main node] (a4) [below of=a2] {$a_4$};
\node[main node] (a5) [left of=a3] {$a_5$};
\node[right of=a2] (f) {$A_F$};

\path[->] (a2) edge (a1)
    (a3) edge (a1)
    (a4) edge (a2)
    (a4) edge (a3);

\path[draw,rounded corners,very thin]
               ([shifttl] a1.north west)
            -- ([shifttr] f.north east)
            -- ([shiftbr] f.south east)
            -- ([shiftbr] a2.south east)
            -- ([shiftbr] a4.south east)
            -- ([shiftbl] a3.south west)
            -- ([shiftbl] a5.south west)
            -- ([shifttl] a5.north west)
            -- ([shifttl] a3.north west)
            -- ([shiftbl] a1.south west)
            -- cycle;

\node[uncertain] (a6) [below of=a4] {$a_6$};
\node[right of=a6] (u) {$A_U$};

\path[->,dotted] (a5) edge[bend left] (a1);
\path[<->,dashed] (a6) edge (a4);

\path[draw,rounded corners,very thin]
                ([shifttr] u.north east)
             -- ([shifttl] a6.north west)
             -- ([shiftbl] a6.south west)
             -- ([shiftbr] u.south east)
             -- cycle ;

\node[control] (a7) [below of=a5] {$a_7$};
\node[left of=a6] (c) {$A_{conf}$};

\path[line width=2pt] (a7) edge (a5);

\path[draw,rounded corners,very thin]
             ([shifttr] c.north east)
             -- ([shifttl] a7.north west)
             -- ([shiftbl] a7.south west)
             -- ([shiftbr] c.south east)
             -- cycle;
\end{tikzpicture}
\caption{$\caf$ configured by $A_{conf} = \{a_7\}$\label{fig:success-configured-caf}}
\end{subfigure}
~\vline~
\begin{subfigure}[t]{0.45\linewidth}
  \centering
\newcommand{\shiftpoints}{5pt} 
\begin{tikzpicture}[->,>=stealth,shorten >=1pt,auto,node distance=1.5cm,
                thick,main node/.style={circle,draw,font=\bfseries}]
\node[main node] (a1) {$a_1$};
\node[main node] (a2) [right of=a1] {$a_2$};
\node[main node] (a3) [below of=a1] {$a_3$};
\node[main node] (a4) [below of=a2] {$a_4$};

\path[->] (a2) edge (a1)
    (a3) edge (a1)
    (a4) edge (a2)
    (a4) edge (a3);

\node[main node] (a5) [left of=a3] {$a_5$};
\node[main node] (a6) [below of=a4] {$a_6$};

\path[->] (a4) edge (a6)
    (a5) edge[bend left] (a1);

\node[main node] (a7) [below of=a5] {$a_7$};

\path[->] (a7) edge (a5);

\end{tikzpicture}
\caption{A successful completion of the CAF \label{fig:success-configured-completion}}
\end{subfigure}
\caption{A configured CAF and a successful completion}
\end{figure}
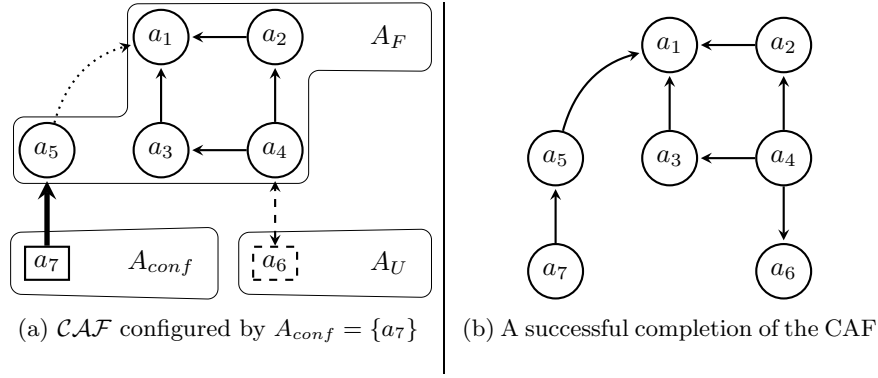
\end{example}

\subsection{Computational Complexity of Possible Controllability}\label{section:complexity}
Now we focus on the computational complexity of deciding whether a CAF
is possibly controllable. Formally, for $x \in \{\sk,\cred\}$ standing
respectively for ``skeptically'' and ``credulously'', and $\sigma \in \{\co, \pr, \gr, \stb\}$, we study the
decision problem:
\begin{description}
\item[$\Control_{\sigma,p,x}^{\caf,T}$] Is the CAF \caf~ possibly $x$-controllable w.r.t. $\sigma$ and $T$?
\end{description}
  
The proofs for hardness rely on complexity results for Incomplete Argumentation Frameworks (IAFs) \cite{BaumeisterNR18}. Let us formally introduce IAFs.

\begin{definition}
An Incomplete Argumentation Framework ({\em IAF}) is a tuple
$\iaf = \langle A, A^?, R, R^?\rangle$ where $A$ and $A^?$ are disjoint sets of arguments, and $R, R^? \subseteq (A \cup A^?) \times (A \cup A^?)$ are disjoint sets of attacks.
\end{definition}

The arguments and attacks in $A^?$ and $R^?$ are uncertain, similarly to the arguments $A_U$ and the attacks $\dashrightarrow$ in a CAF (see Definition~\ref{def:caf}). Thus, an IAF can be associated with a set of completions. This means that different forms of reasoning can be defined, {\em e.g.} the necessary (respectively possible) acceptance of an arguments is the situation where an argument is accepted in each (respectively some) completion. Here, we focus on the possible variants of acceptance, {\em i.e.} Possible Credulous Acceptance (\PCA) and Possible Skeptical Acceptance (\PSA). Formally:
\begin{description}
\item[$\sigma$-$\PCA^{\iaf,a}$] Given $\iaf = \langle A, A^?, R, R^?\rangle$ and $a \in A$ does $a$ belong to some extension of $\iaf$?
\item[$\sigma$-$\PSA^{\iaf,a}$] Given $\iaf = \langle A, A^?, R, R^?\rangle$ and $a \in A$ does $a$ belong to each extension of $\iaf$?
\end{description}

We easily show that the complexity of $\sigma$-$\PCA^{\iaf,a}$ (respectively $\sigma$-$\PSA^{\iaf,a}$) yields lower bounds for the complexity of $\Control_{\sigma,p,\cred}^{\caf,T}$ (resp $\Control_{\sigma,p,\sk}^{\caf,T}$).

\begin{lemma}\label{lemma:hardness}
  Given {\sf C} a complexity class from the polynomial hierarchy,
  \begin{itemize}
  \item if $\sigma$-$\PCA^{\iaf,a}$ is {\sf C}-hard then $\Control_{\sigma,p,\cred}^{\caf,T}$ is {\sf C}-hard;
  \item if $\sigma$-$\PSA^{\iaf,a}$ is {\sf C}-hard then $\Control_{\sigma,p,\sk}^{\caf,T}$ is {\sf C}-hard.
  \end{itemize}
\end{lemma}

\begin{proof}
Let $\iaf = \langle A, A^?, R, R^?\rangle$ be an IAF. We define $\caf$ a CAF such that $A_F = A$, $A_U = A^?$, $A_C = \emptyset$, $\rightarrow = R$, $\dashrightarrow = R^?$, and $\rightleftarrows = \Rrightarrow = \emptyset$. Given an argument $a \in A$, we observe that $\caf$ is possibly credulously (respectively skeptically) controllable with respect to a semantics $\sigma$ and the target $T = \{a\}$ iff $a$ is possibly credulously (respectively skeptically) accepted in $\iaf$. This comes from the fact that $\caf$ does not contain any control argument that would allow to reinstate $\{a\}$ in the case where it is not accepted.
\end{proof}

\begin{proposition}\label{prop:complexity-skeptical}
  \begin{itemize}
  \item For $\sigma \in \{\co, \gr\}$, $\Control_{\sigma,p,\sk}^{\caf,T}$ is \NP-complete.
  \item $\Control_{\stb,p,\sk}^{\caf,T}$ is $\Sigma_2^P$-complete.
  \item $\Control_{\pr,p,\sk}^{\caf,T}$ is $\Sigma_3^P$-complete.
  \end{itemize}
\end{proposition}

\begin{proof}
  Baumeister {\em et~al} \cite{BaumeisterNR18} have proven that $\sigma$-$\PSA^{\iaf,a}$ is \NP-complete, for $\sigma \in \{\co,\gr\}$, $\Sigma_2^P$-complete for the stable semantics, and $\Sigma_3^P$-complete for the preferred semantics. With Lemma~\ref{lemma:hardness}, we obtain the lower bound.

  Then, for proving the upper bound, let us consider the simple non-deterministic algorithm that checks possible skeptical acceptance. Guess a completion of $\caf$, and check whether the target is skeptically accepted. Depending on the complexity of skeptical acceptance in AFs, we obtain different upper bounds for possible skeptical controllability. More precisely, recall that skeptical acceptance is polynomial for the grounded and complete semantics, in \coNP for the stable semantics, and in $\Pi_2^P$ for the preferred semantics. Thus, the algorithm for checking possible skeptical controllability of a CAF allows to deduce the \NP, $\Sigma_2^P$ and $\Sigma_3^P$ upper bounds. This concludes the proof.
\end{proof}

\begin{proposition}\label{prop:complexity-credulous}
For $\sigma \in \{\co, \pr, \gr, \stb\}$, $\Control_{\sigma,p,\cred}^{\caf,T}$ is \NP-complete.
\end{proposition}

\begin{proof}
  Baumeister {\em et~al} \cite{BaumeisterNR18} have proven that $\sigma$-$\PCA^{\iaf,a}$ is \NP-complete, for $\sigma \in \{\co, \pr, \gr, \stb\}$. With Lemma~\ref{lemma:hardness}, we obtain the lower bound. Now let us prove that $\Control_{\sigma,p,\cred}^{\caf,T} \in \NP$. For the grounded semantics, we apply the same method as in the proof of Proposition~\ref{prop:complexity-skeptical}: since credulous acceptance under the grounded semantics is polynomial, we obtain a \NP upper bound the possible credulous controllability.

  If we use the same approach for the other semantics, we obtain higher upper bounds than the ones expected here. However, the computational approach based on QBFs, presented in the next section, shows that possible credulous controllability can actually be reduced to \SAT, which belongs to \NP. The same logic-based approach can be used for the complete semantics, relying on the propositional encoding of this semantics given by Besnard and Doutre \cite{BesnardD04}. Finally, in order to guarantee that a set of arguments is included in a preferred extension, it is enough to guarantee that it belongs to a complete extension (since every complete extension is included in a preferred extension). Thus, the approach for the complete semantics also works for the preferred semantics. This means that possible credulous controllability can also be reduced to \SAT for $\sigma \in \{\co, \pr\}$. This concludes the proof.
\end{proof}

Let us also briefly discuss the complexity of possible controllability for simplified CAFs, defined by \cite{DimopoulosMM18} as CAFs with no uncertainty ({\em i.e.} $A_U = \rightleftarrows = \dashrightarrow = \emptyset$). Such a CAF has only one completion for each control configuration, thus possible and necessary controllability are equivalent in this case, and complexity remains the same as in the general case, described at Table~\ref{table:complexity}.

\subsection{Possible Controllability Through QBFs}\label{section:qbf-computing}
Inspired by \cite{DimopoulosMM18}, we propose a QBF-based method to compute possible controllability for the stable semantics. Let us first give the meaning of the propositional variables used in the encoding.

Given $\af = \langle A, R\rangle$,
\begin{itemize}
 \item $\forall x_i \in A$, $acc_{x_i}$ represents the acceptance status of the argument $x_i$;
 \item $\forall x_i, x_j \in A$, $att_{x_i,x_j}$ represents the attack from $x_i$ to $x_j$.
\end{itemize}
$\Phi_{st}$ is the formula
$\Phi_{st} = \bigwedge_{x_i \in A} [acc_{x_i} \Leftrightarrow
\bigwedge_{x_j \in A}(att_{x_j,x_i} \Rightarrow \neg
acc_{x_j})]$. This modified version of the encoding from
\cite{BesnardD04} describes in a generic way the relation between the
structure of an AF ({\em i.e.} the set of attacks) and the arguments'
acceptance ({\em i.e.} the extensions) w.r.t. stable semantics.

When the $att$-variables are assigned the truth value corresponding to
the attack relation of $\af$ ({\em i.e.} $att_{x_i,x_j}$ is assigned
$1$ iff $(x_i,x_j) \in R$), the models of $\Phi_{st}$
(projected on the $acc$-variables) correspond in a bijective way to
 $\stb(\af)$. 

 Given $\af = \langle A, R\rangle$, we define the formula
\[
  \Phi_{st}^R = \Phi_{st} \wedge (\bigwedge_{(x_i,x_j) \in R}
att_{x_i,x_j}) \wedge (\bigwedge_{(x_i,x_j) \notin R} \neg
att_{x_i,x_j})
\]
which represents this assignment of $att$-variables
corresponding to a specific AF. For any model $\omega$ of
$\Phi_{st}^R$, the set $\{x_i \mid \omega(acc_{x_i}) = 1\}$ is a
stable extension of $\af$. In the other direction, for any stable
extension $\varepsilon \in \stb(\af)$, $\omega$ s.t.  $\omega(acc_{x_i}) = 1$
iff $x_i \in \varepsilon$ is a model of $\Phi_{st}^R$.

These variables and formula are enough to encode the stable semantics of
AFs. But to determine the controllability of a CAF, we need also to
consider propositional variables to indicate which arguments are actually in the
system:
\begin{itemize}
\item $\forall x_i \in A_C \cup A_U$, $on_{x_i}$ is true iff $x_i$ actually appears in the framework.
\end{itemize}

Now, we can recall the encoding which relates the attack relation and
the arguments statuses in $\caf = \langle F, C, U\rangle$ \cite{DimopoulosMM18}:
\ \\
{\bf Notation:} $\A = A_F \cup A_C \cup A_U$, $\R = \rightarrow \cup \rightleftarrows \cup \dashrightarrow \cup \Rrightarrow$
\[
\begin{array}{c}
\Phi_{st}(\caf) = \bigwedge_{x_i \in A_F} [acc_{x_i} \Leftrightarrow
  \bigwedge_{x_j \in \A}(att_{x_j,x_i} \Rightarrow \neg acc_{x_j})]
  \wedge \\
  \bigwedge_{x_i \in A_C \cup A_U} [acc_{x_i} \Leftrightarrow
(on_{x_i} \wedge
  \bigwedge_{x_j \in \A}(att_{x_j,x_i} \Rightarrow \neg acc_{x_j}))] \wedge \\
(\bigwedge_{(x_i,x_j) \in \rightarrow \cup \Rrightarrow} att_{x_i,x_j}) \wedge
(\bigwedge_{(x_i,x_j) \in \rightleftarrows} att_{x_i,x_j} \vee
  att_{x_j,x_i}) \\
  \wedge (\bigwedge_{(x_i,x_j) \notin \R} \neg att_{x_i,x_j})
\end{array}
\]

The first line states that an argument from $A_F$ is accepted when all
its attackers are rejected (similarly to the case of classical
AFs). Then, the next line concerns arguments from $A_C$ and $A_U$;
since these arguments may not appear in some completions of the CAF,
we add the condition that $on_{x_i}$ is true to allow $x_i$ to be
accepted.  The last line specify the case in which there is an attack
in the completion: attacks from $\rightarrow$ and $\Rrightarrow$ are
mandatory, and their direction is known; attacks from
$\rightleftarrows$ are mandatory, but the actual direction is not
known. We do not give any constraint about $\dashrightarrow$, which is
equivalent to the tautological constraint
$att_{x_i,x_j} \vee \neg att_{x_i,x_j}$: the attack may appear or
not. Finally, we know that attacks which are not in $\R$ do not exist.

Given a set of arguments $T$, the fact that $T$ must be included in
all the stable extensions is represented by:
\[
\Phi_{st}^{\sk}(\caf, T) = \Phi_{st}(\caf) \Rightarrow \bigwedge_{x_i \in T} acc_{x_i}
\]

Given a set of arguments $T$, the fact that $T$ must be included in
at least one stable extension is represented by:
\[
\Phi_{st}^{\cred}(\caf, T) = \Phi_{st}(\caf) \wedge \bigwedge_{x_i \in T} acc_{x_i}
\]

Now we give the logical encodings for possible controllability for $\sigma = \stb$.

\begin{proposition}\label{prop:skeptical-encoding}
Given $\caf$ and $T \subseteq A_F$, $\caf$ is possibly skeptically
controllable w.r.t. $T$ and the stable semantics iff
\begin{equation}
\begin{array}{l}
  \exists \{on_{x_i} \mid x_i \in A_C\}\exists \{on_{x_i} \mid x_i \in A_U\}\\
  \exists\{att_{x_i,x_j} \mid (x_i,x_j) \in \dashrightarrow \cup \rightleftarrows\}\forall\{acc_{x_i} \mid x_i \in \A\}
  \\
  \lbrack\Phi_{st}^{\sk}(\caf, T)\vee (\bigvee_{(x_i,x_j)
  \in \rightleftarrows} (\neg att_{a_i,a_j} \wedge \neg att_{a_j,a_i}))]
\end{array}
\label{eq:skepticalEncoding}
\end{equation}
is valid. In this case, each valid truth assignment of the variables
$\{on_{x_i} \mid x_i \in A_C\}$ corresponds to a configuration which
reaches the target.
\end{proposition}

This encoding is not a direct adaptation of the encoding proposed in \cite{DimopoulosMM18}. We have to explicitly exclude the joint assignment of the variables $att_{x_i,x_j}$ and $att_{x_j,x_i}$ to false, when $(x_i,x_j) \in \rightleftarrows$, which would be in contradiction with the definition of this conflict relation. Another method is used in \cite{DimopoulosMM18} to rule out these assignments, but it does not yield a QBF in prenex form. But this is the method that was proposed in \cite{DimopoulosMM19}, when necessary controllability has been applied to automated negotiation.

The following result holds for possible credulous controllability:

\begin{proposition}
Given $\caf$ and $T \subseteq A_F$, $\caf$ is possible credulously
controllable w.r.t. $T$ and the stable semantics iff
\begin{equation}
\begin{array}{l}
  \exists \{on_{x_i} \mid x_i \in A_C\}\exists \{on_{x_i} \mid x_i \in A_U\}\\
  \exists\{att_{x_i,x_j} \mid (x_i,x_j) \in \dashrightarrow \cup \rightleftarrows\}  \exists\{acc_{x_i} \mid x_i \in \A\}\\
  \lbrack\Phi_{st}^{\cred}(\caf, T)  \vee (\bigvee_{(x_i,x_j)
  \in \rightleftarrows} (\neg att_{a_i,a_j} \wedge \neg att_{a_j,a_i}))]
\end{array}
\label{eq:credulousEncoding}
\end{equation}
is valid. In this case, each valid truth assignment of the variables
$\{on_{x_i} \mid x_i \in A_C\}$ corresponds to a configuration which
reaches the target.
\end{proposition}

We notice that in the case of possible credulous controllability, the
problem reduces to \SAT~ since all the quantifiers are existential. This corresponds to the \NP{} upper bound for possible credulous controllability under stable semantics (Proposition~\ref{prop:complexity-credulous}). We
keep the QBF-style notation for homogeneity
with Equation~\ref{eq:skepticalEncoding}.

\begin{example}[Continuation of Example~\ref{example:running-caf}]
Let us describe the logical encoding for possible controllability with
\caf{} as described previously and $T = \{a_1\}$. We give here the
example for possible skeptical controllability:
\[
\begin{array}{l}
  \exists on_{a_7}, on_{a_8}, on_{a_9}, \exists on_{a_6},
  \exists att_{a_5,a_1}, att_{a_6,a_5}, att_{a4,a_6} ,\\
  \forall acc_{a_1}, acc_{a_2}, \dots, acc_{a_9},
  \\
  \lbrack\Phi_{st}^{\sk}(\caf, T)\vee (\bigvee_{(x_i,x_j) \in \rightleftarrows} (\neg att_{a_i,a_j} \wedge \neg
  att_{a_j,a_i})) \rbrack
\end{array}
\]
Below, we give the formula $\Phi_{st}^{\sk}(\caf, T)$. For a matter of
readability, several simplifications are made. For instance, an
implication like $att_{x_j,x_i} \Rightarrow \neg acc_{x_j}$ can be
removed when $att_{x_j,x_i}$ is known to be false (because $x_j$ does
not attack $x_i$), and can be replaced by $\neg acc_{x_j}$ when
$att_{x_j,x_i}$ is known to be true. Only the uncertain attacks need
to be kept explicit in the encoding. The first three lines give the
condition for the acceptance of the fixed arguments. Then, two lines
give the condition for the acceptance of the control and uncertain
arguments. The other lines describe the structure of the graph ({\em
  i.e.} the attack relations), and the implication gives the target
for skeptical acceptance.

\[
\begin{array}{c}
\lbrack 
\lbrack acc_{a_1} \Leftrightarrow
  \neg acc_{a_2} \wedge \neg acc_{a_3} \wedge (att_{a_5,a_1} \Rightarrow \neg acc_{a_5})] \\
  \wedge \\
\lbrack acc_{a_2} \Leftrightarrow \neg acc_{a_4}] 
  \wedge 
\lbrack acc_{a_3} \Leftrightarrow \neg acc_{a_4}] \\
  \wedge \\
\lbrack acc_{a_4} \Leftrightarrow (att_{a_6,a_4} \Rightarrow \neg acc_{a_6})] 
  \wedge 
\lbrack acc_{a_5} \Leftrightarrow \neg acc_{a_7})] \\

  \wedge \\

[acc_{a_6} \Leftrightarrow
(on_{a_6} \wedge \neg acc_{a_8} \wedge \neg acc_{a_9} \wedge (att_{a_4,a_6} \Rightarrow \neg acc_{a_4}))] \\
  \wedge \\

[acc_{a_7} \Leftrightarrow
(on_{a_7} \wedge \neg acc_{a_8})] 
  \wedge 

[acc_{a_8} \Leftrightarrow on_{a_8}] 
  \wedge 

[acc_{a_9} \Leftrightarrow
(on_{a_9} \wedge \neg acc_{a_7}))] \\
  \wedge \\

  att_{a_2,a_1} \wedge  att_{a_3,a_1} \wedge  att_{a_4,a_2} \wedge
  att_{a_4,a_3} \wedge  att_{a_7,a_5} \\
  \wedge\\
  att_{a_7,a_9} \wedge  att_{a_8,a_6} \wedge  att_{a_8,a_7} \wedge  att_{a_9,a_6} \\

\wedge\\

 (att_{a_4,a_6} \vee  att_{a_6,a_4}) \wedge \bigwedge_{(x_i,x_j) \notin \R} \neg att_{x_i,x_j} ] \Rightarrow acc_{a_1}
\end{array}
\]
\end{example}

\section{Related Work}\label{section:related-work}
Qualitative uncertainty has been considered in other
frameworks. Partial AFs \cite{Coste-MarquisDKLM07} are special
instances of CAFs where only $\dashrightarrow$ is considered. They
are used as a tool in a process of aggregating several AFs. Then
\cite{BaumeisterNR15} studies the complexity of verifying in a PAF
whether a set of arguments is an extension of some (or every)
completion. \cite{BaumeisterRS15} conducts a similar study for
argument-incomplete AFs, {\em i.e.} there is some uncertainty about
the presence of arguments (the part called $A_U$)in our
framework). Finally, \cite{BaumeisterNRS18b} combines both. Let us
notice than in \cite{BaumeisterNR15,BaumeisterRS15,BaumeisterNRS18b},
both versions of the verification problem (existential and universal
w.r.t. the set of completions) are studied. As mentioned previously, \cite{BaumeisterNR18}  gives the complexity of skeptical and credulous acceptance for IAFs. While being a quite
general model of uncertainty, this Incomplete AF is strictly included in the CAF setting:
\cite{BaumeisterRS15} does not allow to express the uncertainty about
the direction of a conflict ({\em i.e.}  our $\rightleftarrows$
relation cannot be encoded in this framework). Moreover, none of these
works \cite{Coste-MarquisDKLM07,BaumeisterNR15,BaumeisterRS15,
  BaumeisterNRS18b} is concerned with argumentation dynamics.

Quantitative models of uncertainty have also been used; while being an
interesting approach, they require more input information than
qualitative models like ours. This approach is out of the scope of
this paper and is kept for future work. In particular, probabilistic
CAFs based on the constellations approach \cite{Hunter13} are a
promising research tracks.

Argumentation dynamics has received a lot of attention in the last ten
years. Except the initial paper about CAFs \cite{DimopoulosMM18}, most
of the existing work consider complete information about the input
({\em i.e.} no uncertainty of the initial AF is considered). As far as
we know, the only proposal which can encompass uncertainty is the
update of AFs through the YALLA language
\cite{Saint-CyrBCL16}. However, YALLA pays the price of its
expressiveness, and we are not aware of any efficient computational
approach for reasoning with it, contrary to our QBF-based
approach for CAFs. On the opposite, the recent work by Niskanen {\em et al.} \cite{NiskanenNJ20} has given a full picture of complexity for necessary controllability of CAFs, as well as QBF-based and SAT-based algorithms that have been experimentally evaluated.

\section{Conclusion}\label{section:conclusion}

In this paper, we push forward the study of the Control Argumentation
Frameworks. We define a ``weaker'' version of controllability, where a
target set of arguments needs to be accepted in at least one
completion (instead of every completion). This kind of reasoning is
related to a lawyer's plea: at the end of a trial, the lawyer needs to
pick arguments (in our setting, the configuration $A_{conf}$) such
that the target (``the defendant is innocent'') is accepted in at
least one completion. Somehow, possible controllability is to
necessary controllability what credulous acceptance is to skeptical
acceptance.

We have studied the computational complexity of this new form of reasoning, for the four classical Dung semantics, namely the stable, complete, grounded and preferred semantics. We recall our results in Table~\ref{table:complexity}.

\begin{table}[h]
\centering
\begin{tabular}{c|c|c}
	$\sigma$ & $\sk$ & $\cred$ \\ \hline
	$\stb$ & $\Sigma_2^P$-complete & \NP-complete \\
	$\co$ & \NP-complete & \NP-complete \\
	$\gr$ & \NP-complete & \NP-complete \\
	$\pr$ & $\Sigma_3^P$-complete & \NP-complete
\end{tabular}
\caption{The complexity of $\Control_{\sigma,p,x}^{\caf,T}$, for $x \in \{\sk, \cred\}$\label{table:complexity}}
\end{table}

Many research tracks are still open. We plan to propose logical encodings and to study the complexity of controllability for other extension-based semantics. Also,  other methods can be used for computing control configuration, especially SAT-based counter-example guided abstract refinement (CEGAR), that was successfully used for reasoning problems at the second level of the polynomial hierarchy \cite{WallnerNJ17,NiskanenNJ20}. 
An interesting other form of controllability to be studied is ``optimal'' controllability, {\em i.e.} finding a configuration that allows to reach the target in as many completions as possible. This is useful in situations where a CAF is not necessary controllable, and possible controllability seems too weak. Techniques like CEGAR or QBF with soft variables \cite{ReimerSMB14} may be helpful for solving this problem.
Also, as mentioned previously, we will study quantitative models of
uncertainty in the context of CAFs. In particular, it would be
interesting for real world applications to define a form of
controllability w.r.t. the most probable completion, or w.r.t. the set
of completions with a probability higher than a given threshold.
Finally, we think that an important work to be done, in order to apply
CAFs to real applications scenarios, is to determine how CAFs and
controllability can be defined when the internal structure of
arguments ({\em e.g.} based on logical formulas or rules) is known.



\bibliography{possible_controllability}

\begin{thebibliography}{10}

\bibitem{AmgoudB13a}
Leila Amgoud and Jonathan Ben{-}Naim.
\newblock Ranking-based semantics for argumentation frameworks.
\newblock In {\em Proc. of {SUM}'13}, pages 134--147, 2013.

\bibitem{AmgoudV14}
Leila Amgoud and Srdjan Vesic.
\newblock Rich preference-based argumentation frameworks.
\newblock {\em Int. J. Approx. Reasoning}, 55(2):585--606, 2014.

\bibitem{AroraB09}
Sanjeev Arora and Boaz Barak.
\newblock {\em Computational Complexity - {A} Modern Approach}.
\newblock Cambridge University Press, 2009.

\bibitem{HOFASemantics}
Pietro Baroni, Martin Caminada, and Massimiliano Giacomin.
\newblock Abstract argumentation frameworks and their semantics.
\newblock In Pietro Baroni, Dov Gabbay, Massimiliano Giacomin, and Leendert
  van~der Torre, editors, {\em Handbook of Formal Argumentation}, pages
  159--236. College Publications, 2018.

\bibitem{Baumann12}
Ringo Baumann.
\newblock What does it take to enforce an argument? minimal change in abstract
  argumentation.
\newblock In {\em Proc. of {ECAI}'12}, pages 127--132, 2012.

\bibitem{BaumannB10}
Ringo Baumann and Gerhard Brewka.
\newblock Expanding argumentation frameworks: Enforcing and monotonicity
  results.
\newblock In {\em Proc. of {COMMA}'10}, pages 75--86, 2010.

\bibitem{BaumeisterNR15}
Dorothea Baumeister, Daniel Neugebauer, and J{\"{o}}rg Rothe.
\newblock Verification in attack-incomplete argumentation frameworks.
\newblock In {\em Proc. of {ADT}'15}, pages 341--358, 2015.

\bibitem{BaumeisterNR18}
Dorothea Baumeister, Daniel Neugebauer, and J{\"{o}}rg Rothe.
\newblock Credulous and skeptical acceptance in incomplete argumentation
  frameworks.
\newblock In {\em Proc. of {COMMA}'18}, pages 181--192, 2018.

\bibitem{BaumeisterNRS18b}
Dorothea Baumeister, Daniel Neugebauer, J{\"{o}}rg Rothe, and Hilmar Schadrack.
\newblock Verification in incomplete argumentation frameworks.
\newblock {\em Artif. Intell.}, 264:1--26, 2018.

\bibitem{BaumeisterRS15}
Dorothea Baumeister, J{\"{o}}rg Rothe, and Hilmar Schadrack.
\newblock Verification in argument-incomplete argumentation frameworks.
\newblock In {\em Proc. of {ADT}'15}, pages 359--376, 2015.

\bibitem{BesnardD04}
Philippe Besnard and Sylvie Doutre.
\newblock Checking the acceptability of a set of arguments.
\newblock In {\em Proc. of {NMR}'04}, pages 59--64, 2004.

\bibitem{BiereHVMW09}
Armin Biere, Marijn Heule, Hans van Maaren, and Toby Walsh, editors.
\newblock {\em Handbook of Satisfiability}, volume 185 of {\em Frontiers in
  Artificial Intelligence and Applications}. {IOS} Press, 2009.

\bibitem{BoellaKT09a}
Guido Boella, Souhila Kaci, and Leendert W.~N. van~der Torre.
\newblock Dynamics in argumentation with single extensions: Abstraction
  principles and the grounded extension.
\newblock In {\em Proc. of {ECSQARU}'09}, pages 107--118, 2009.

\bibitem{Caminada06}
Martin Caminada.
\newblock On the issue of reinstatement in argumentation.
\newblock In {\em Proc. of {JELIA}'06}, pages 111--123, 2006.

\bibitem{CayrolSL10}
Claudette Cayrol, Florence~Dupin de~Saint{-}Cyr, and Marie{-}Christine
  Lagasquie{-}Schiex.
\newblock Change in abstract argumentation frameworks: Adding an argument.
\newblock {\em J. Artif. Intell. Res. {(JAIR)}}, 38:49--84, 2010.

\bibitem{Coste-MarquisDKLM07}
Sylvie Coste{-}Marquis, Caroline Devred, S{\'{e}}bastien Konieczny,
  Marie{-}Christine Lagasquie{-}Schiex, and Pierre Marquis.
\newblock On the merging of dung's argumentation systems.
\newblock {\em Artif. Intell.}, 171(10-15):730--753, 2007.

\bibitem{Coste-MarquisKMM14}
Sylvie Coste{-}Marquis, S{\'{e}}bastien Konieczny, Jean{-}Guy Mailly, and
  Pierre Marquis.
\newblock On the revision of argumentation systems: Minimal change of arguments
  statuses.
\newblock In {\em Proc. of {KR}'14}, 2014.

\bibitem{Coste-MarquisKM15}
Sylvie Coste{-}Marquis, S{\'{e}}bastien Konieczny, Jean{-}Guy Mailly, and
  Pierre Marquis.
\newblock Extension enforcement in abstract argumentation as an optimization
  problem.
\newblock In {\em Proc. of {IJCAI}'15}, pages 2876--2882, 2015.

\bibitem{Saint-CyrBCL16}
Florence~Dupin de~Saint{-}Cyr, Pierre Bisquert, Claudette Cayrol, and
  Marie{-}Christine Lagasquie{-}Schiex.
\newblock Argumentation update in {YALLA} (yet another logic language for
  argumentation).
\newblock {\em Int. J. Approx. Reasoning}, 75:57--92, 2016.

\bibitem{DimopoulosMM18}
Yannis Dimopoulos, Jean-Guy Mailly, and Pavlos Moraitis.
\newblock Control argumentation frameworks.
\newblock In {\em Proc. of AAAI'18}, pages 4678--4685, 2018.

\bibitem{DimopoulosMM19}
Yannis Dimopoulos, Jean{-}Guy Mailly, and Pavlos Moraitis.
\newblock Argumentation-based negotiation with incomplete opponent profiles.
\newblock In {\em Proc. of {AAMAS}'19}, pages 1252--1260, 2019.

\bibitem{DoutreHP14}
Sylvie Doutre, Andreas Herzig, and Laurent Perrussel.
\newblock A dynamic logic framework for abstract argumentation.
\newblock In {\em Proc. of {KR}'14}, 2014.

\bibitem{DoutreM17}
Sylvie Doutre and Jean{-}Guy Mailly.
\newblock Semantic change and extension enforcement in abstract argumentation.
\newblock In {\em Proc. of {SUM}'17}, pages 194--207, 2017.

\bibitem{Dung95}
Phan~Minh Dung.
\newblock On the acceptability of arguments and its fundamental role in
  nonmonotonic reasoning, logic programming and n-person games.
\newblock {\em Art. Intel.}, 77:321--357, 1995.

\bibitem{Hunter13}
Anthony Hunter.
\newblock A probabilistic approach to modelling uncertain logical arguments.
\newblock {\em Int. J. Approx. Reasoning}, 54(1):47--81, 2013.

\bibitem{QBFHandbook}
Hans Kleine{ }B{\"{u}}ning and Uwe Bubeck.
\newblock Theory of quantified boolean formulas.
\newblock In {\em Handbook of Satisfiability}, pages 735--760. 2009.

\bibitem{NiskanenNJ20}
Andreas Niskanen, Daniel Neugebauer, and Matti J{\"{a}}rvisalo.
\newblock Controllability of control argumentation frameworks.
\newblock In {\em Proceedings of the Twenty-Ninth International Joint
  Conference on Artificial Intelligence, {IJCAI} 2020}, pages 1855--1861, 2020.

\bibitem{ReimerSMB14}
Sven Reimer, Matthias Sauer, Paolo Marin, and Bernd Becker.
\newblock {QBF} with soft variables.
\newblock {\em {ECEASST}}, 70, 2014.

\bibitem{WallnerNJ17}
Johannes~Peter Wallner, Andreas Niskanen, and Matti J{\"{a}}rvisalo.
\newblock Complexity results and algorithms for extension enforcement in
  abstract argumentation.
\newblock {\em J. Artif. Intell. Res.}, 60:1--40, 2017.

\end{thebibliography}

\end{document}